\def\year{2022}\relax
\documentclass[letterpaper]{article} % DO NOT CHANGE THIS
\usepackage{aaai22}  % DO NOT CHANGE THIS
\usepackage{times}  % DO NOT CHANGE THIS
\usepackage{helvet}  % DO NOT CHANGE THIS
\usepackage{courier}  % DO NOT CHANGE THIS
\usepackage[hyphens]{url}  % DO NOT CHANGE THIS
\usepackage{graphicx} % DO NOT CHANGE THIS
\usepackage{natbib}  % DO NOT CHANGE THIS AND DO NOT ADD ANY OPTIONS TO IT
\usepackage{caption} % DO NOT CHANGE THIS AND DO NOT ADD ANY OPTIONS TO IT
\DeclareCaptionStyle{ruled}{labelfont=normalfont,labelsep=colon,strut=off} % DO NOT CHANGE THIS
\usepackage{algorithm}
\usepackage{algorithmic}

\usepackage{newfloat}
\usepackage{listings}
\floatstyle{ruled}
\newfloat{listing}{tb}{lst}{}
\floatname{listing}{Listing}

\usepackage{amsmath}
\usepackage{amssymb}
\usepackage{amsthm}
\usepackage{booktabs}
\usepackage{mathtools}
\usepackage{subcaption}

\newtheorem{theorem}{Theorem}[]

\title{Robust Optimal Classification Trees Against Adversarial Examples}
\author {
    Dani\"el Vos,
    Sicco Verwer
}
\affiliations {
    Delft University of Technology \\
    d.a.vos@tudelft.nl, s.e.verwer@tudelft.nl
}

\begin{document}

\maketitle

\begin{abstract}
Decision trees are a popular choice of explainable model, but just like neural networks, they suffer from adversarial examples. Existing algorithms for fitting decision trees robust against adversarial examples are greedy heuristics and lack approximation guarantees. In this paper we propose ROCT, a collection of methods to train decision trees that are optimally robust against user-specified attack models. We show that the min-max optimization problem that arises in adversarial learning can be solved using a single minimization formulation for decision trees with 0-1 loss. We propose such formulations in Mixed-Integer Linear Programming and Maximum Satisfiability, which widely available solvers can optimize. We also present a method that determines the upper bound on adversarial accuracy for any model using bipartite matching. Our experimental results demonstrate that the existing heuristics achieve close to optimal scores while ROCT achieves state-of-the-art scores.
\end{abstract}

\section{Introduction}
Since the discovery of adversarial examples in neural networks \cite{szegedy2013intriguing} much work has gone into training models that are robust to these attacks. Recently, efforts were made to train robust decision trees against adversarial examples.
Chen et al. \cite{chen2019robust} trained robust trees by scoring with the worst-case information gain or Gini impurity criteria and proposed a fast approximation. Calzavara et al. \cite{calzavara2020treant} created TREANT, a flexible approach in which the user describes a threat model using rules and a loss function under attack that it greedily optimizes. GROOT \cite{vos2020efficient} sped up the worst-case Gini impurity computation and defended against user-specified box-shaped attack models. All these algorithms are greedy and therefore only make locally optimal decisions.

In decision tree learning, there has been an increased interest in optimal learning algorithms~\cite{carrizosa2021mathematical}. Although the problem of learning decision trees is NP-complete~\cite{laurent1976constructing}, these methods can produce optimally accurate decision trees for many (typically small) datasets. Most methods translate the problem to well-known frameworks such as Mixed-Integer Linear Programming~\cite{bertsimas2017optimal,verwer2017learning}, Boolean Satisfiability~\cite{narodytska2018learning,avellaneda2020efficient}, and Constraint Programming~\cite{verhaeghe2020learning}.

In this work we combine these lines of research and propose Robust Optimal Classification Trees (ROCT), a method to train decision trees that are optimally robust against user-specified adversarial attack models. Like existing robust decision tree learning algorithms~\cite{calzavara2020treant,vos2020efficient}, ROCT allows users to specify
a box-shaped attack model that encodes an attacker's capability to modify feature values with the aim of maximizing loss. 
Existing robust decision tree learning methods use a greedy node splitting approach.
Other robust learning algorithms such as adversarial training~\cite{madry2017towards} solve the inner maximization (adversarial attacks) and the outer minimization problems (minimize expected loss) separately. In this work we demonstrate that this separation is not needed in the case of decision trees. We provide a formulation that solves the problem of fitting robust decision trees exactly in a single minimization step for a tree up to a given depth.

ROCT uses a novel translation of the problem of fitting robust decision trees into a Mixed-Integer Linear Programming (MILP) or Maximum Satisfiability (MaxSAT) formulations. We also propose a new upper-bound calculation for the adversarial accuracy of any machine learning model based on bipartite matching which can be used to choose appropriate attack models for experimentation.
Our results show that ROCT trees optimized with a warm-started MILP solver achieve state-of-the-art adversarial accuracy scores compared to existing methods on 8 datasets.
Moreover, given sufficient solver time, ROCT provably finds an optimally robust decision tree. In our experiments, ROCT was able to fit and prove optimality of depth $2$ decision trees on six datasets.
Where there are no known approximation bounds on the performance of existing heuristic methods for fitting robust decision trees, our results demonstrate that they are empirically close to optimal.

\section{Background and Related Work}

\subsection{Mixed-Integer Linear Programming}
Mixed-Integer Linear Programming (MILP) is a variation of Linear Programming in which some variables are integers or binary. The goal of these formulations is to optimize a linear function under linear constraints. While the problem is generally NP-hard there exist many fast MILP solvers. In this paper, we translate the robust decision tree learning problem into a MILP formulation and solve it using GUROBI\footnote{\url{https://www.gurobi.com/}} 9. Since MILP solvers usually become less efficient with more integer variables, we introduce two formulations that differ in the number of such variables.

\subsection{Maximum Satisfiability Solving}
Maximum Satisfiability (MaxSAT) is an optimization version of the classical boolean satisfiability problem (SAT).
In MaxSAT, problems are modeled as a set of hard clauses that have to be satisfied and a set of soft clauses of which the solver tries to satisfy as many as possible. One advantage of MaxSAT solvers is their availability, with many state-of-the-art solvers available as open source programs. In this work we use the PySat implementation \cite{imms-sat18} of the Linear Sat-Unsat (LSU) algorithm \cite{morgado2013iterative} improved with incremental cardinality constraints \cite{martins2014incremental}, and the RC2 algorithm \cite{ignatiev2019rc2}. These algorithms differ in the direction in which they optimize,
LSU starts with a poor solution and creates increasingly optimal solutions over time while RC2 starts by attempting to satisfy all soft clauses then relaxes this constraint until it finds a solution.
Both algorithms use the Glucose\footnote{\url{https://www.labri.fr/perso/lsimon/glucose/}} 4.1 SAT solver.

\subsection{Optimal Decision Trees}
Most popular decision tree learning algorithms such as CART \cite{breiman1984classification}, ID3 \cite{quinlan1986induction} and C4.5 \cite{quinlan1993c4} are greedy and can return 
arbitrarily bad trees \cite{kearns1996boosting}. In recent years, there has been extensive effort to train optimal trees.
One of the earliest works~\cite{bertsimas2007classification}
proposes a MILP formulation for finding a tree of a given maximum depth that uses clustering to reduce the dataset size. Independently,~\cite{nijssen2010optimal} maps this problem for the restricted case of Boolean decision nodes to itemset mining. Several years later the first MILP formulations were proposed for the full problem
~\cite{bertsimas2017optimal} and 
~\cite{verwer2017learning}. The latest methods improve these works using a binary encoding~\cite{verwer2019learning}, translation to CP~\cite{verhaeghe2020learning}, dynamic programming with search~\cite{demirovic2020murtree}, and caching branch-and-bound~\cite{aglin2020learning}.
In this work, we build on these works to create the first formulation for optimal learning of robust decision trees.

\subsection{Robust Decision Trees}
Previous works have already put effort into fitting decision trees that are more robust to adversarial perturbations than the trees created by regular decision tree algorithms. \cite{kantchelian2016evasion} defines a MILP formulation for finding adversarial examples in decision tree ensembles and used these samples to fit an ensemble of more robust decision trees. \cite{chen2019robust} adapts greedy decision tree learning algorithms by using the worst case score functions under attacker influence to fit more robust trees against $L_\infty$ norm bounded attackers. Later, TREANT \cite{calzavara2020treant} uses a more flexible greedy algorithm that could optimize arbitrary convex score functions under attacker influence and allowed users to describe attacker capabilities using axis-aligned rules. This flexibility comes at a cost in run-time, as it uses an iterative solver to optimize this score for each split it learns. GROOT \cite{vos2020efficient} improve the greedy procedure 
by efficiently computing the worst-case Gini impurity and allowing users to specify box-shaped attacker perturbation limits. In this paper, we compare against GROOT and TREANT as these greedy methods achieve state-of-the-art scores.

\section{ROCT: Robust Optimal Classification Trees}
When training robust classifiers we find ourselves in a competition with the adversary. Madry et al. \cite{madry2017towards} present the robust learning problem as the following min-max optimization problem:
\begin{equation} \label{eq:robust-learning-problem}
    \min_\theta \mathbb{E}_{(x, y) \sim D} \left( \max_{\delta \in S} L(\theta, x + \delta, y) \right)
\end{equation}

Here the learning algorithm attempts to find model parameters $\theta$ that minimize the expected risk (outer minimization) after an attacker moves samples from distribution $D$ within perturbation domain $S$ (inner maximization). Intuitively, the min-max nature of training robust models makes it a much more challenging optimization problem than regular learning. For example in adversarial training \cite{madry2017towards} one approximately optimizes this function by incorporating expensive adversarial attacks into the training procedure. In this work, we demonstrate that this intuition does not hold when learning decision trees.
We present Robust Optimal Classification Trees (ROCT), which turns Equation \ref{eq:robust-learning-problem} into a single minimization problem that can be solved using combinatorial optimization. ROCT consists of 6 methods that are based on the same intuitions but vary in the kind of solver (MILP or MaxSAT) and type of variables used to represent splitting thresholds, see Table \ref{tab:roct-overview}.

\begin{table}[tb]
\caption{Summary of introduced methods, they differ in solver type and whether thresholds are formulated with binary or continuous variables. `Warm' methods are initialized with the GROOT heuristic.}
\begin{tabular}{@{}c|ccc@{}}
\toprule
Method & \begin{tabular}[c]{@{}c@{}}Threshold\\ formulation\end{tabular} & Solver & \begin{tabular}[c]{@{}l@{}}Init. with\\ GROOT\end{tabular} \\ \midrule
LSU-MaxSAT & binary & \begin{tabular}[c]{@{}c@{}}LSU\\ (glucose 4.1)\end{tabular} &  \\ \midrule
RC2-MaxSAT & binary & \begin{tabular}[c]{@{}c@{}}RC2\\ (glucose 4.1)\end{tabular} &  \\ \midrule
Binary-MILP & binary & GUROBI 9 &  \\ \midrule
\begin{tabular}[c]{@{}c@{}}Binary-MILP-\\ warm\end{tabular} & binary & GUROBI 9 & \checkmark \\ \midrule
MILP & continuous & GUROBI 9 &  \\ \midrule
MILP-warm & continuous & GUROBI 9 & \checkmark \\ \bottomrule
\end{tabular}
\label{tab:roct-overview}
\end{table}

\begin{theorem}
Robust learning (Equation \ref{eq:robust-learning-problem}) with 0-1 loss
in the case of binary classification trees is equivalent to:
\[
\min_\theta \sum_{(x, y) \sim D} \left[ \bigvee_{t \in \mathcal{T}_L^S} c_t \neq y \right]
\]
where $\mathcal{T}_L^S$ is the set of leaf nodes that intersect with the possible perturbations $S$, and $c_t$ is the class label of leaf $t$.
\end{theorem}

\begin{proof}
For 0-1 loss $L_{\text{0-1}}$, Equation \ref{eq:robust-learning-problem} is equivalent to:
\[
\min_\theta \sum_{(x, y) \sim D} \left( \max_{\delta \in S} L_{\text{0-1}}(\theta, x + \delta, y) \right)
\]
A decision tree $\mathcal{T}$ maps any data point $x$ to a leaf node $t = \mathcal{T}(x)$ and assigns $c_t$ as its prediction. Any perturbation in the inner maximization $\max_{\delta \in S}$ such that $\mathcal{T}(x) = \mathcal{T}(x + \delta)$ gives the same classification outcome for 0-1 loss. The maximization over all $\delta \in S$ can therefore be replaced by a maximization over all leaf nodes $t \in \mathcal{T}_L^S$ that can be reached by any permutation in $S$.
By definition, the 0-1 loss term is equivalent to the absolute difference $|c_t - y|$ of prediction $c_t$ and label $y$, which gives:
\[
\min_\theta \sum_{(x, y) \sim D} \left( \max_{t \in \mathcal{T}_L^S} |c_t - y| \right)
\]
The term $|c_t - y|$ takes value $1$ when $c_t \neq y$ and $0$ otherwise. When any of the reachable leafs $t \in \mathcal{T}_L^S$ returns $c_t \neq y$, the inner maximization becomes 1. This 
is equivalent to the disjunction over $c_t \neq y$ for all reachable leafs:
\[
\min_\theta \sum_{(x, y) \sim D} \left[ \bigvee_{t \in \mathcal{T}_L^S} c_t \neq y \right]
\]
\end{proof}
The resulting formulation can be solved in one shot using discrete optimization solvers such as MaxSAT and MILP.

\subsection{Attack Model} \label{sec:threat-model}
We assume the existence of a white-box adversary that can move all samples within a box-shaped region centered around each sample. This box-shaped region can be defined by two vectors $\Delta^l$ and $\Delta^r$ specifying for each feature by how much the feature value can be decreased and increased respectively. For the ease of our formulation we scale all feature values to be in the range $[0, 1]$ which means that the values in $\Delta^l$ and $\Delta^r$ encode distance as a fraction of the feature range. While our encoding is more flexible, we only test on attack models that encode an $L_\infty$ radius. This allows us to easily evaluate performance against a variety of attacker strengths by only changing the perturbation radius $\epsilon$.

\begin{table}[tb]
\caption{Summary of the notation used throughout the paper.}
\setlength{\tabcolsep}{4.5pt}
\centering
\begin{tabular}{lll}  
\toprule
Symbol & Type & Definition \\
\midrule
$a_{jm}$ & variable & node $m$ splits on feature $j$ \\
$b_{vm}$ & variable & node $m$'s threshold is left/right of $v$ \\
$b'_{m}$ & variable & node $m$'s continuous threshold value \\
$c_{t}$ & variable & leaf node $t$ predicts class $0$ or $1$ \\
$s_{im0}$ & variable & sample $i$ can move left of node $m$ \\
$s_{im1}$ & variable & sample $i$ can move right of node $m$ \\
$e_{i}$ & variable & sample $i$ can be misclassified \\
\midrule
$X_{ij}$ & constant & value of data row $i$ in feature $j$ \\
$y_{i}$ & constant & class label of data row $i$ \\
$\Delta^l_{j}$ & constant & left perturbation range for feature $j$ \\
$\Delta^r_{j}$ & constant & right perturbation range for feature $j$ \\
$n$ & constant & number of samples \\
$p$ & constant & number of features \\
\midrule
$A(t)$ & set & ancestors of node $t$ \\
$A_l(t)$ & set & ... with left branch on the path to $t$ \\
$A_r(t)$ & set & ... with right branch on the path to $t$ \\
$\mathcal{T}_B$ & set & all decision nodes \\
$\mathcal{T}_L$ & set & all leaf nodes \\
$V_j$ & set & unique values in feature $j$ \\
\bottomrule
\end{tabular}
\label{tab:notation}
\end{table}

\begin{figure*}[tb]
     \centering
     \begin{subfigure}[b]{.32\textwidth}
         \centering
         \includegraphics[width=\textwidth]{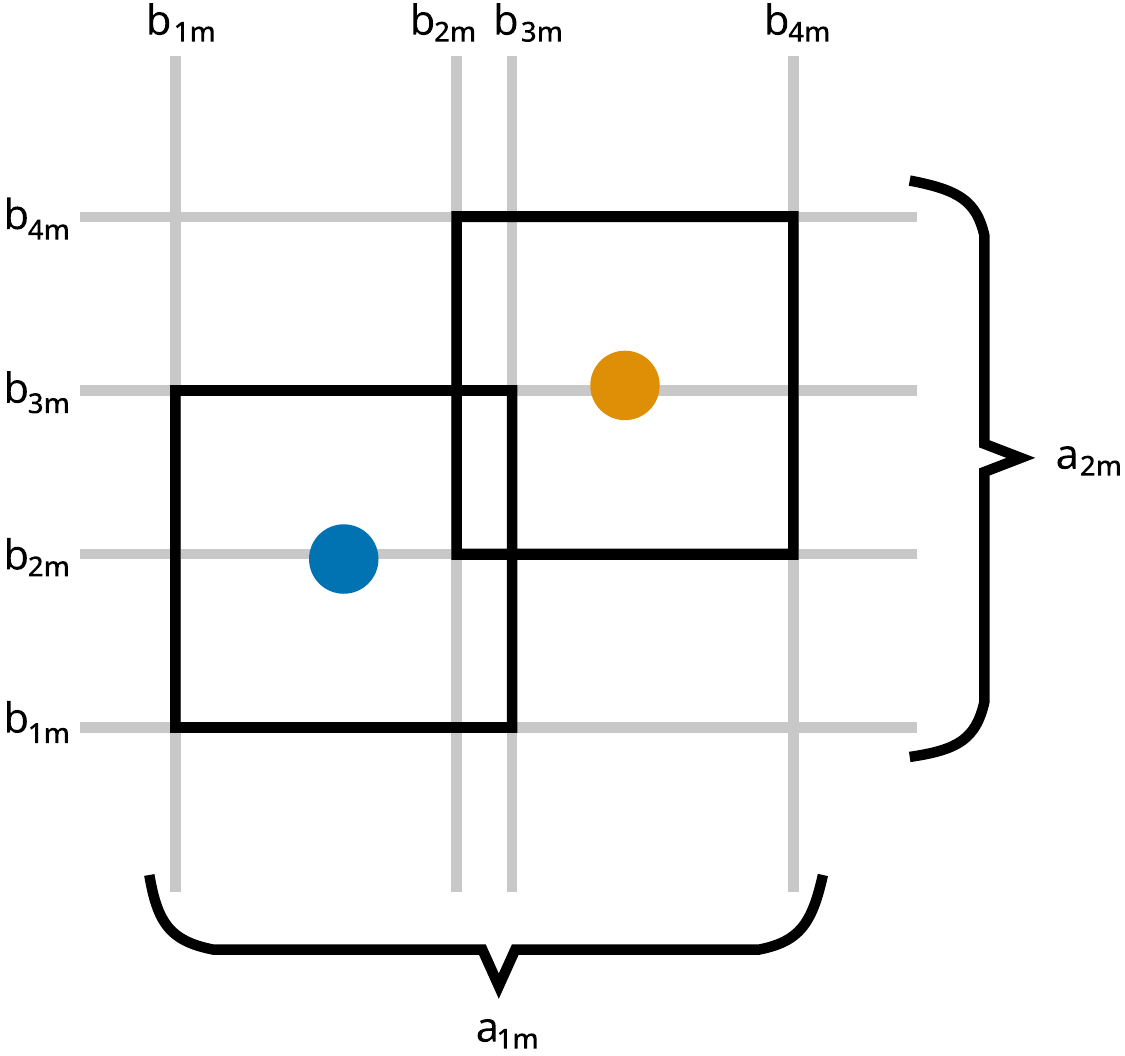}
         \caption{Decision node $m$ (binary)}
     \end{subfigure}
     \begin{subfigure}[b]{.32\textwidth}
         \centering
         \includegraphics[width=\textwidth]{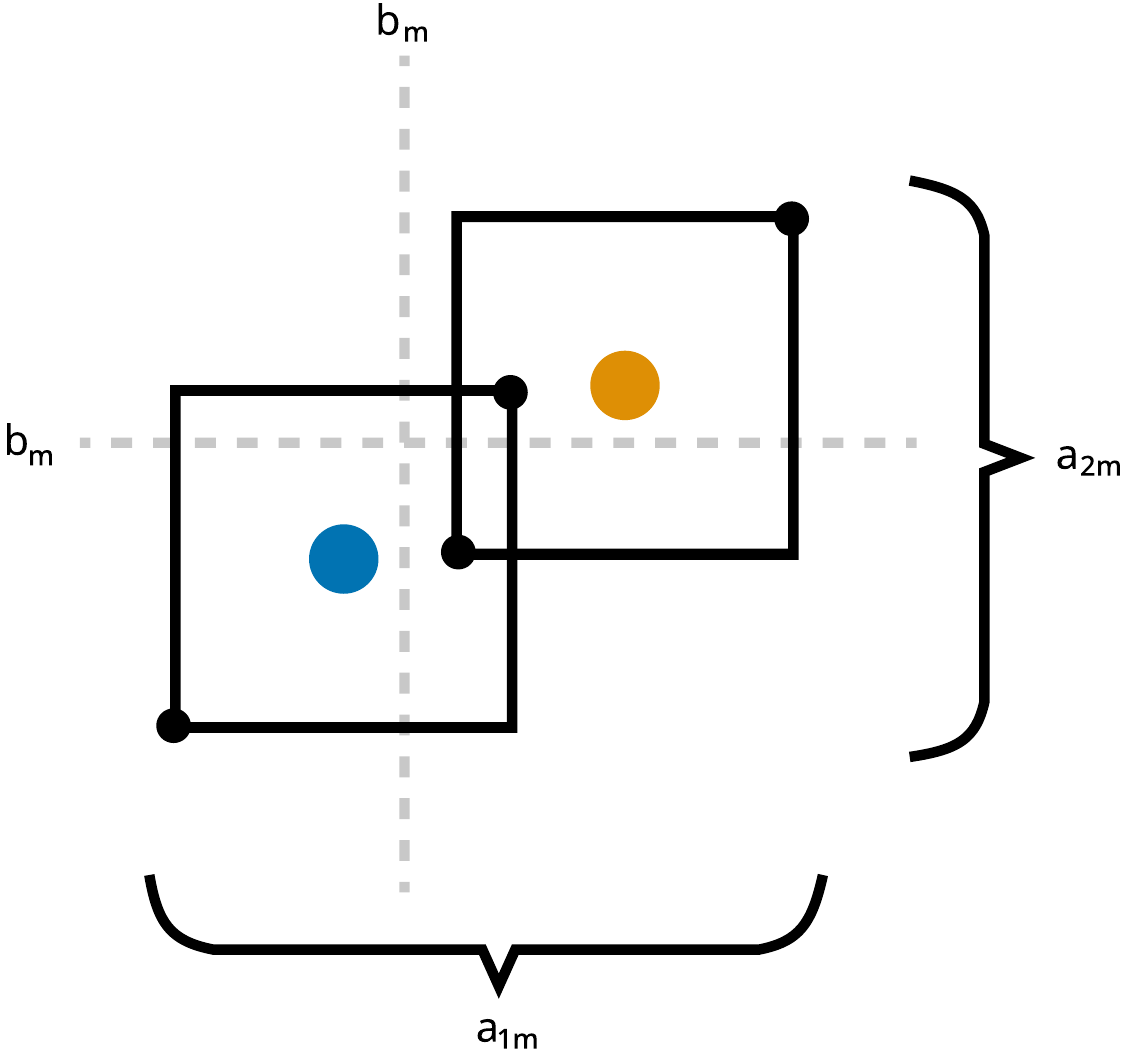}
         \caption{Decision node $m$ (continuous)}
     \end{subfigure}
     \hfill
     \begin{subfigure}[b]{.33\textwidth}
         \centering
         \includegraphics[width=\textwidth]{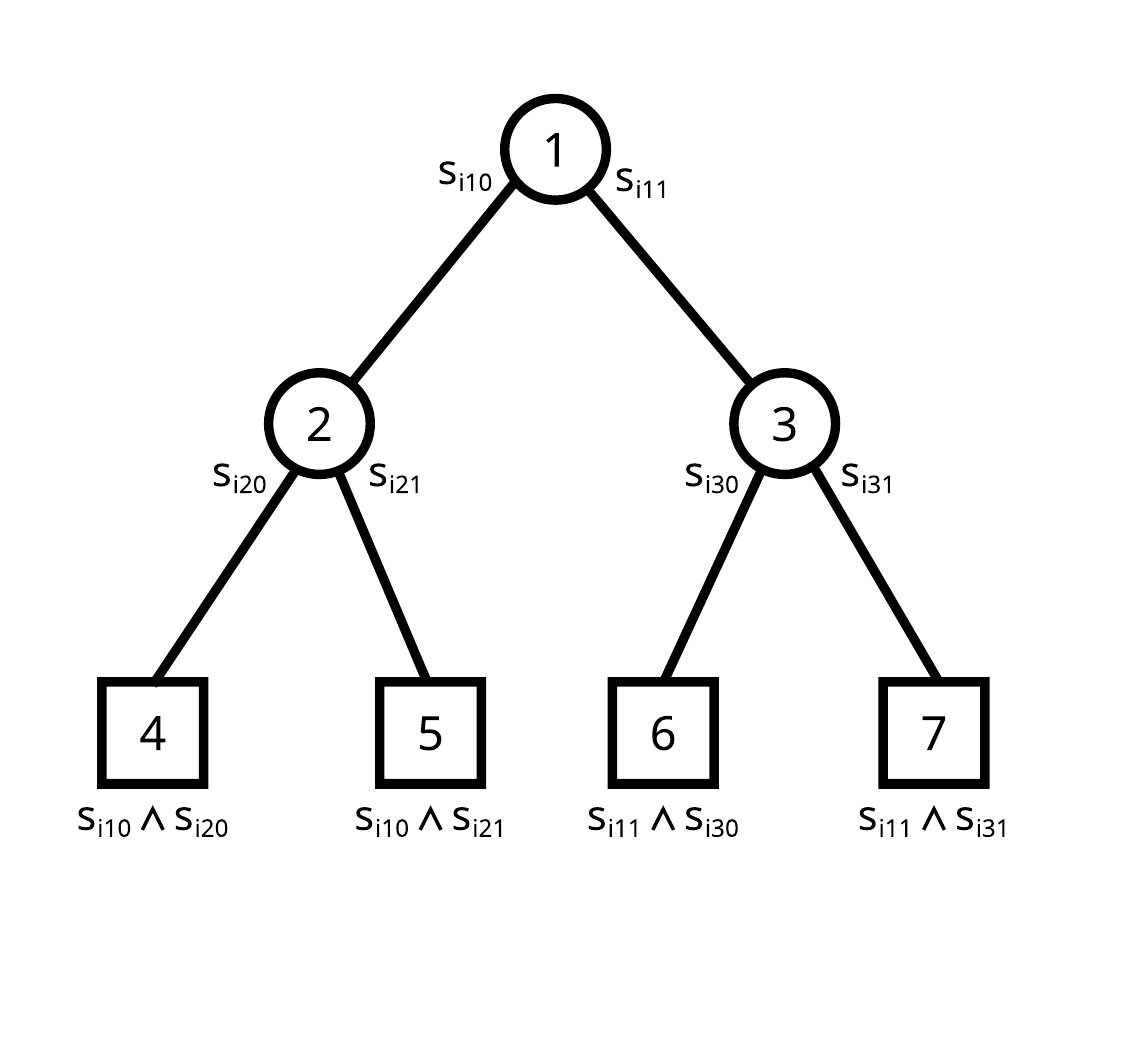}
         \caption{Path variables for sample $i$}
     \end{subfigure}
    \caption{Example of ROCT's formulation. For each decision node the $a$ variables select a splitting feature and $b$ select the threshold value. $b$ can be defined as multiple binary (a) or a single continuous (b) variable. Using the $s$ variables (c) ROCT traces all sample paths through the tree to the leaves and counts an error if any reachable leaf predicts the wrong class.}
    \label{fig:formulation-visualized}
\end{figure*}

\subsection{Intuition}
We borrow much of the notation from OCT \cite{bertsimas2017optimal}, summarized in Table \ref{tab:notation}. Figure \ref{fig:formulation-visualized} visualizes the variables in ROCT and Figure \ref{fig:simple-example} shows an example of the constraints for a single sample and a tree of depth 1.
In the regular learning setting where samples cannot be perturbed by an adversary, samples can only propagate to the left or right child of decision node. In the adversarial setting, samples can permute and are able to reach both the left and right sides, i.e. $s_{im0}$ and $s_{im1}$ can be true at the same time.

Given the attacker capabilities $\Delta^l$ and $\Delta^r$, we create the constraints to set the variables $s$. To determine whether sample $X_i$ can move left of the chosen split we can decrease its feature values as far as the attacker capabilities allow ($X_i - \Delta^l$) and see if it reaches the left side. Similarly to see if it reaches the right side we increase the feature values maximally ($X_i + \Delta^r$). We give two kinds of constraints for determining these $s$ variables that differ in whether decision thresholds are represented by binary or continuous variables.

\subsubsection{Continuous Decision Thresholds}
To select a threshold value an intuitive method is to create a continuous variable $b_m$ for every decision node. We can then use this variable to determine the values $s_{im0}$ and $s_{im1}$ by checking whether $X_i - \Delta^l$ and $X_i + \Delta^r$ can reach the left and right side of the threshold respectively. We create the following constraints:
\begin{equation*}
    \mathbf{X_i} \cdot \mathbf{a_{m}} \leq b'_m \implies s_{im0}
\end{equation*}
\begin{equation*}
    \mathbf{X_i} \cdot \mathbf{a_{m}} > b'_m \implies s_{im1}
\end{equation*}
Since these constraints use a dot product with continuous variables it is not possible to implement this in MaxSAT. Another challenge comes with the second constraint being a strict inequality which is not directly supported in MILP. Like \cite{bertsimas2017optimal}, we add a small value to the right hand side to turn it into a regular inequality.

\subsubsection{Binary Decision Thresholds}
We create a set of variables $b_{vm}$ for each unique decision threshold value $v$, with $v$ in ascending order. Instead of forcing one of them to \texttt{true}, we create an ordering in the variables such that if one threshold variable is \texttt{true}, the larger variables also become \texttt{true}:
\begin{equation*}
    b_{vm} \implies b_{(v+1)m}
\end{equation*}
Intuitively if $A(b_{vm})$ is \texttt{true} a sample with feature value $v$ will be sent to the right of the split and when $A(b_{vm})$ is \texttt{false} it will be sent to the left.
A useful property of this constraint is that we only have to encode the local influence of a threshold variable $b_{vm}$ on close-by data points, the rest is forced by the chain of constraints.
For each feature $j$ we determine what threshold values $v^l$ and $v^r$ correspond to $X_{ij} - \Delta^l_j$ and $X_{ij} + \Delta^r_j$ and check whether their $b_{vm}$ values indicate that the sample can reach the left / right side:
\begin{equation*}
    a_{jm} \land \neg b_{v^lm} \implies s_{im0}
\end{equation*}
\begin{equation*}
    a_{jm} \land b_{v^rm} \implies s_{im1}
\end{equation*}

\subsubsection{Selecting Features}
Consider a single decision node $m$, such a decision node needs to decide a feature to split on. We create a binary variable $a_{jm}$ for each feature $j$ and force that exactly one of these variables can be equal to $1$:
\begin{equation*}
    \sum_{j=1}^p a_{jm} = 1
\end{equation*}
This constraint can be relaxed to $\sum_{j=1}^p a_j \geq 1$ as selecting more than one feature can only make more $s$ variables \texttt{true} and thus can only increase the number of errors.

\subsubsection{Counting errors}
We create a variable $e_i$ for each sample $i$ which is true when any reachable leaf $t \in \mathcal{T}_L^S$ (see Theorem 1)  predicts the other class.
These leaves are found by following all paths a sample can take through the tree using the $s_{im0}$ and $s_{im1}$ variables.
This is visualized in Figure \ref{fig:formulation-visualized}c.
Sample $i$ can reach leaf $t$ when the values $s_{im...}$ are \texttt{true} for all nodes $m$ on the path to $t$, i.e. $\bigwedge_{m \in A_l(t)} s_{im0} \bigwedge_{m \in A_r(t)} s_{im1}$. Here $A_l(t)$ refers to the set of ancestors of leaf $t$ of which we follow the path through its left child and $A_r(t)$ for child nodes on the right. When sample $i$ can reach leaf $t$ and its label does not match $t$'s prediction ($y_i \neq c_t$), force $e_i$ to \texttt{true}:
\begin{equation*}
    \bigwedge_{m \in A_l(t)} s_{im0} \bigwedge_{m \in A_r(t)} s_{im1} \land (c_{t} \neq y_i) \implies e_i
\end{equation*}
With one constraint per decision leaf and sample combination this determines the $e$ values.
To then turn all possible paths into predictions we need to assign a prediction label to each decision leaf. Each leaf $t$ gets a variable $c_t$ where \texttt{false} means class 0 and \texttt{true} means class 1.

\subsubsection{Objective Function}
Our goal is to minimize the equation from Theorem 1.
This is equivalent to minimizing the sum of errors $e_i$ ($i = 1 ... n$). We convert this MILP objective to MaxSAT by adding a soft constraint $\neg e_i$ for each sample and maximizing the number of correctly predicted samples:
\begin{equation*}
    \text{maximize} \quad \sum_{i=1}^n \neg e_i \quad \text{or} \quad \text{minimize} \quad \sum_{i=1}^n e_i
\end{equation*}

\subsection{Complete Formulation} \label{sec:formulation}
Below we give the full formulation for ROCT, in Table \ref{tab:notation} we summarize the notation used. The 
equations can easily be formulated as MILP or MaxSAT instances, these conversions are given in the appendix.\\

\small
\begin{align*}
& \text{min.} \sum_{i = 1}^n e_i && \\
& \text{subject to:} && \\
& \sum_{j = 1}^p a_{jm} = 1, && \forall m \in \mathcal{T}_B \\
& b_{vm} \Rightarrow b_{(v+1)m}, && \forall m \in \mathcal{T}_B, v{=}1 .. |V_j|-1 \\
& \;\;\;\bigwedge_{\mathclap{m \in A_l(t)}} s_{im0} \bigwedge_{\mathclap{m \in A_r(t)}} s_{im1} \land [c_{t} {\neq} y_i] \Rightarrow e_i, && \forall t \in \mathcal{T}_L, i{=}1..n \\
& \textbf{continuous threshold variables:} && \\
& \mathbf{X_i} \cdot \mathbf{a_{m}} \leq b'_m \Rightarrow s_{im0} && \forall m \in \mathcal{T}_B, i{=}1..n \\
& \mathbf{X_i} \cdot \mathbf{a_{m}} > b'_m \Rightarrow s_{im1} && \forall m \in \mathcal{T}_B, i{=}1..n \\
& \textbf{binary threshold variables:} && \\
& a_{jm} \land \neg b_{v^lm} \Rightarrow s_{im0}, && \forall m \in \mathcal{T}_B, i{=}1 .. n, j{=}1 .. p \\
& a_{jm} \land b_{v^rm} \Rightarrow s_{im1}, && \forall m \in \mathcal{T}_B, i{=}1..n, j{=}1..p
\end{align*}
\normalsize

\begin{figure}[tb]
     \centering
     \includegraphics[width=.98\linewidth]{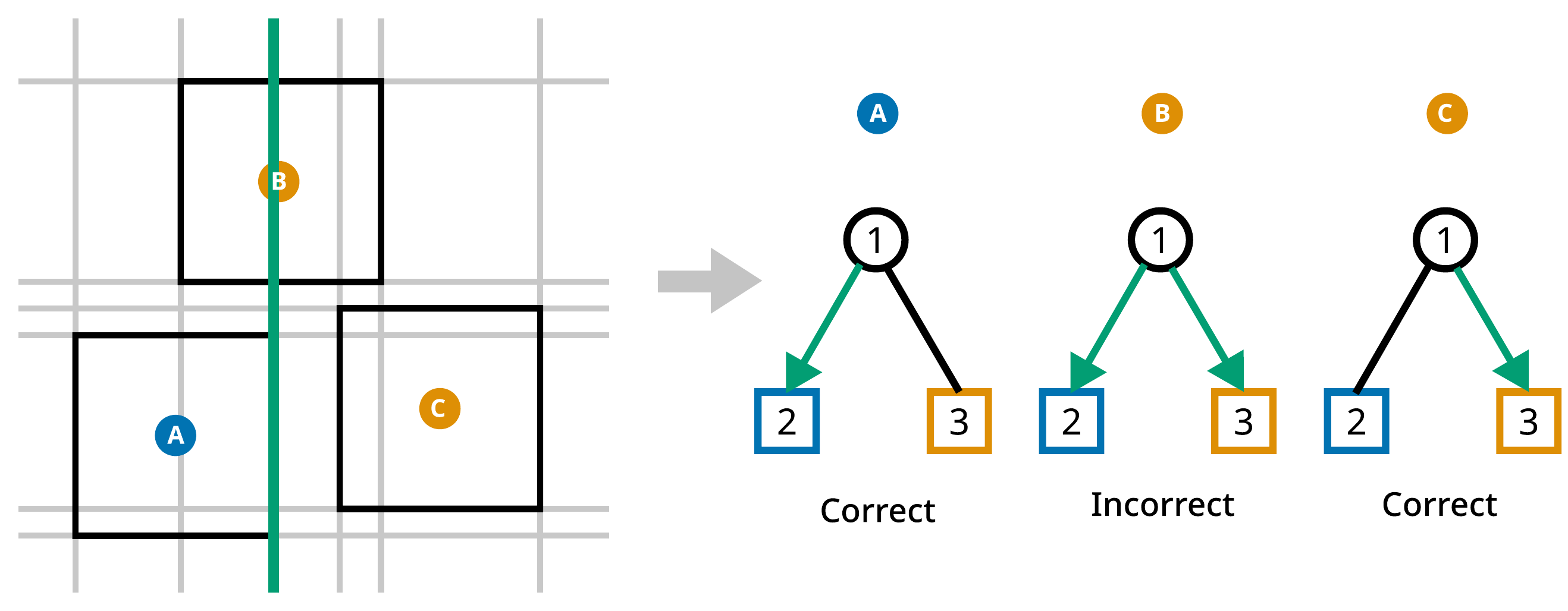}
    \caption{Example of a decision tree of depth 1 with the binary threshold formulation. Sample A and C get correctly classified since all their reachable leaves predict the correct label. Sample B reaches both leaves, since the left leaf predicts the wrong label, B gets misclassified.}
    \label{fig:simple-example}
\end{figure}

\subsubsection{Example}
For clarity we give a small example of a decision tree of depth 1 that we fit on 3 samples. In Figure \ref{fig:simple-example} the solver selects a threshold that forces two samples into a leaf with their correct label but leaves one sample close enough to the split that it can move both ways. Since we only show a single decision node, we drop $m$ from all subscripts of variables. The assignments $A$ to the variables are
$A(a_{1}) = 1$ (selecting feature $1$), $A(b_1), \ldots, A(b_6) = 0,0,0,1,1,1$ (selecting the 3rd out of 6 possible thresholds), the $s$ variables become:
\[
A(b_{1}) = 0 \Rightarrow A(s_{1,0}) = 1
~~~
A(b_{3}) = 0 \Rightarrow A(s_{1,1}) = 0
\]
\[
A(b_{2}) = 0 \Rightarrow A(s_{2,0}) = 1
~~~
A(b_{5}) = 1 \Rightarrow A(s_{2,1}) = 1
\]
\[
A(b_{4}) = 1 \Rightarrow A(s_{3,0}) = 0
~~~
A(b_{6}) = 1 \Rightarrow A(s_{3,1}) = 1
\]
Since $A(c_1) = y_1$ and $A(c_2) = y_3$,  $e_1$ and $e_3$ are unconstrained and minimized to 0 by the solver, and since $A(s_{2,0}) = A(s_{2,1}) = 1$ the constraints force $A(e_2) = 1$. The second sample is hence misclassified (it reaches at least one leaf with a prediction value different than its label).
Note that, although the thresholds in Figure \ref{fig:simple-example} are always exactly on the perturbation ranges of a sample, we post-process these to maximize the margin.

\begin{figure}[tb]
    \centering
    \includegraphics[width=.99\linewidth]{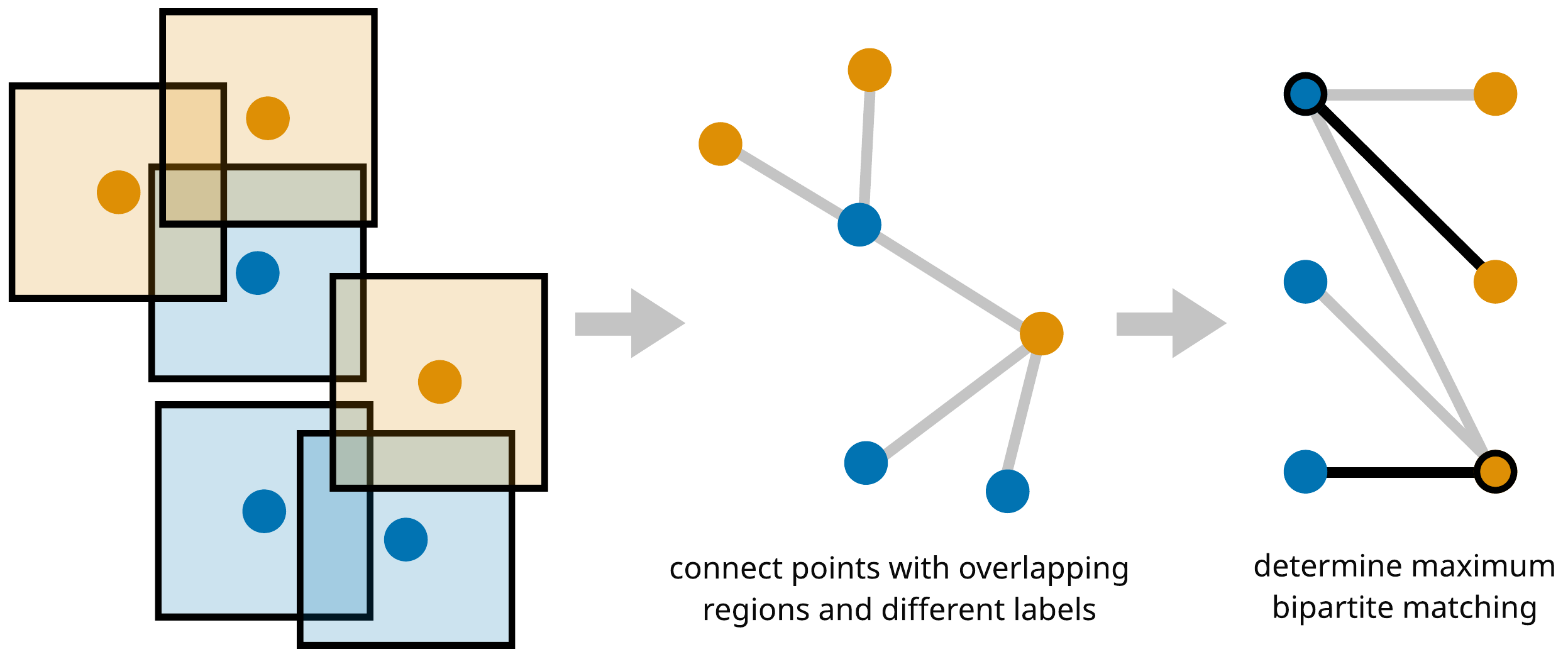}
    \caption{Computing a bound on adversarial accuracy by maximum matching. The maximum matching and minimum vertex cover are shown in black. Since the matching has a cardinality of 2 it is impossible to misclassify fewer than 2 samples when accounting for perturbations.}
    \label{fig:bipartite-matching}
\end{figure}

\section{Upper Bound on Adversarial Accuracy}
In a regular learning setting with stationary samples one strives for a predictive accuracy of 100\%. As long as there are no data points with different labels but same coordinates achieving this score is theoretically possible. However, we realize that in the adversarial setting a perfect classifier cannot always score 100\% accuracy as samples can be perturbed. We present a method to compute the upper bound on adversarial accuracy using a bipartite matching that can be computed regardless of what model is used.
We use this bound to choose better $\epsilon$ values for our experiments.
It also lets us compare the scores of optimal decision trees to a score that is theoretically achievable by perfect classifiers. Such a matching approach was also used in \cite{wang2018analyzing} to train robust kNN classifiers.

\begin{theorem}
The maximum cardinality bipartite matching between samples with overlapping perturbation range and different labels $\{ (i,j) : S_i \cap S_j \neq \emptyset \land y_i \neq y_j \}$ gives an upper bound to the adversarial accuracy achievable by any model for binary decision problems.
\end{theorem}
\begin{proof} The reduction to maximum bipartite matching is based on the realization that when the perturbation ranges of two samples with different labels overlap it is not possible to predict both of these samples correctly. A visual explanation is given in Figure \ref{fig:bipartite-matching}. Formally, given a classifier $C$ that maps samples to a class $0$ or $1$, a sample $i$ can only be correctly predicted against an adversary if its entire perturbation range $S_i$ is correctly predicted: 
\begin{equation} \label{eq:adv-sample-correct}
    \forall x \in S_i: C(x) = y_i
\end{equation}

Now given a sample $j$ of a different class (e.g. $y_i = 0$ and $y_1 = 1$) that has an overlapping perturbation range such that $S_i \cap S_j \neq \emptyset$, it is clear that Equation \ref{eq:adv-sample-correct} cannot simultaneously hold for both samples. We create a bipartite graph $G = (V_0, V_1, E)$ with $V_0 = \{i : y_i = 0 \}$ and $V_1 = \{i : y_i = 1\}$, i.e., vertices representing samples of class $0$ on one side and class $1$ on the other. We then connect two vertices with an edge if their perturbation ranges overlap and their labels are different: $E = \{ (i,j) : S_i \cap S_j \neq \emptyset \land y_i \neq y_j \}$.

To obtain the upper bound, we consider the minimum vertex cover $V'$ from $G$.
By removing all vertices / samples in $V'$, none of the remaining samples can be transformed to have identical feature values with a sample from the opposite class. A perfect classifier $C'(x)$ would therefore assign these rows their correct class values and an attacker will not be able to influence the score of this classifier. 
It is not possible to misclassify fewer samples than the cardinality of the minimum vertex cover $V'$ since removing any vertex from it will add at least one edge $e \in \{ (i,j) : S_i \cap S_j \neq \emptyset \land y_i \neq y_j \}$ which will cause an additional misclassification.
By K\"onig's theorem such a minimum cover in a bipartite graph is equivalent to a maximum matching. Therefore we can use a maximum matching solver to compute an upper bound on the adversarial accuracy.
\end{proof}

\begin{figure}[tb]
    \centering
    \includegraphics[width=\linewidth]{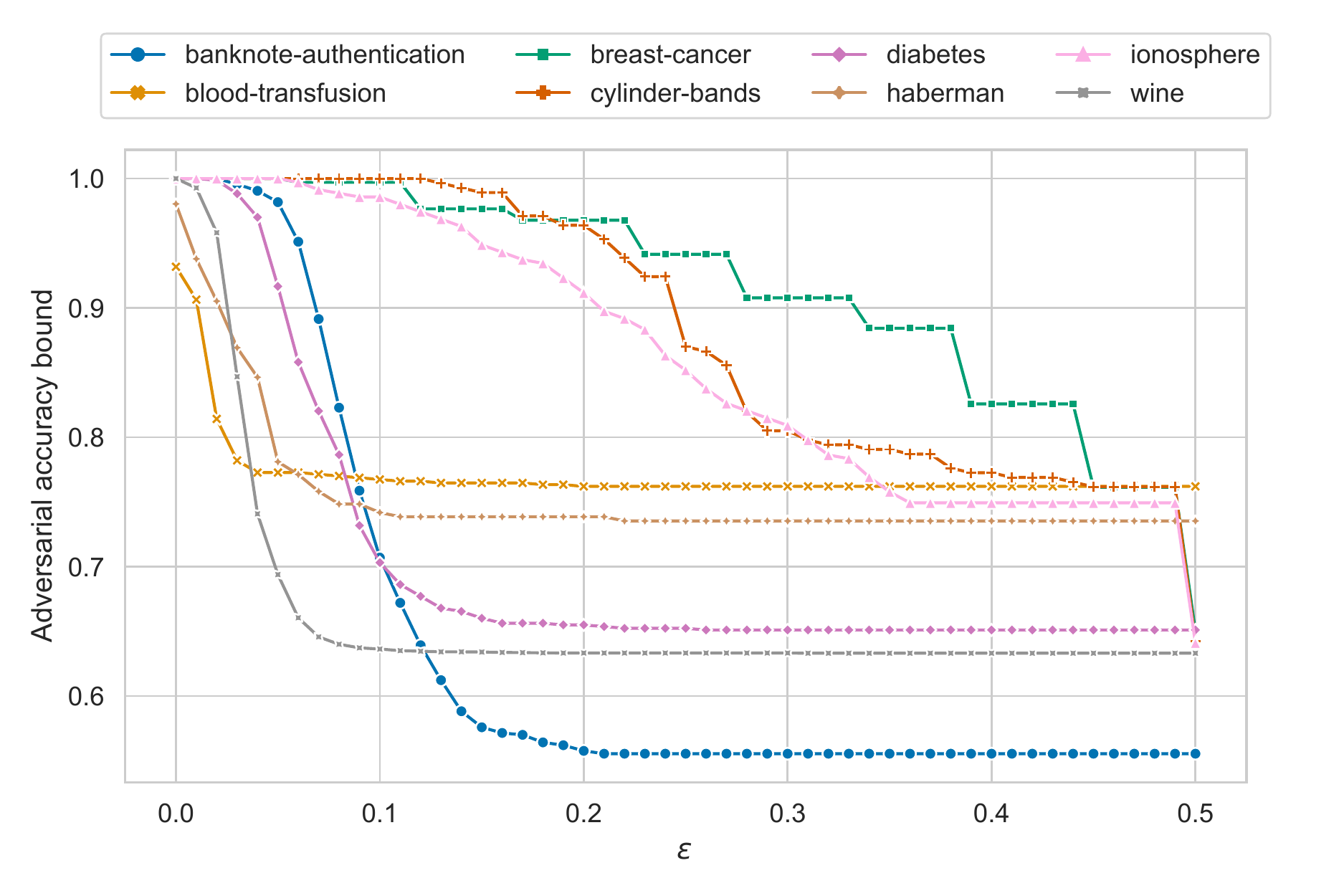}
    \caption{Varying the $L_\infty$ perturbation radius $\epsilon$ and computing the adversarial accuracy bound. Datasets are affected differently, e.g. $\epsilon{=}0.1$ has no effect on cylinder-bands while the bound for blood-transfusion shows that it is not possible to score better than constantly predicting its majority class.}
    \label{fig:epsilons}
\end{figure}

\subsection{Improving Experiment Design}
In previous works \cite{calzavara2020treant,vos2020efficient} attacker capabilities were arbitrarily chosen but this limits the value of algorithm comparisons, shown in Figure \ref{fig:epsilons}.
In this figure we vary the $L_\infty$ radius $\epsilon$ by which an adversary can perturb samples.
Particularly, if this value is chosen too large, the best possible model is a trivial one that constantly predict the majority class. If $\epsilon$ is chosen too small, the adversary has no effect on the learning problem.

To improve the design of our experiments we propose to choose values for $\epsilon$ along these curves that cause the adversarial accuracy bound to be non-trivial. In our experiments we choose three $\epsilon$ values for each dataset such that their values corresponds to an adversarial accuracy bound that is at $25\%$-$50\%$-$75\%$ of the range.

\section{Results}
To demonstrate the effectiveness of ROCT we compare it to the state-of-the-art robust tree learning algorithms TREANT and GROOT. First we run the algorithms on an artificial XOR dataset to show that the heuristics can theoretically learn arbitrarily bad trees, see Figure~\ref{fig:bad-example}. Then to compare the practical performance we run the algorithms on eight popular datasets \cite{chen2019robust,vos2020efficient} and varying perturbation radii ($\epsilon$).
All of our experiments ran on 15 Intel Xeon CPU cores and 72 GB of RAM total, where each algorithm ran on a single core.
We give an overview of the datasets and epsilon radii in the Appendix. These datasets are used in many of the existing works to compare robust tree learning algorithms and are available on OpenML\footnote{\url{http://www.openml.org}}.

\begin{figure}[tb]
     \centering
     \begin{subfigure}[b]{.15\textwidth}
         \centering
         \includegraphics[width=\textwidth]{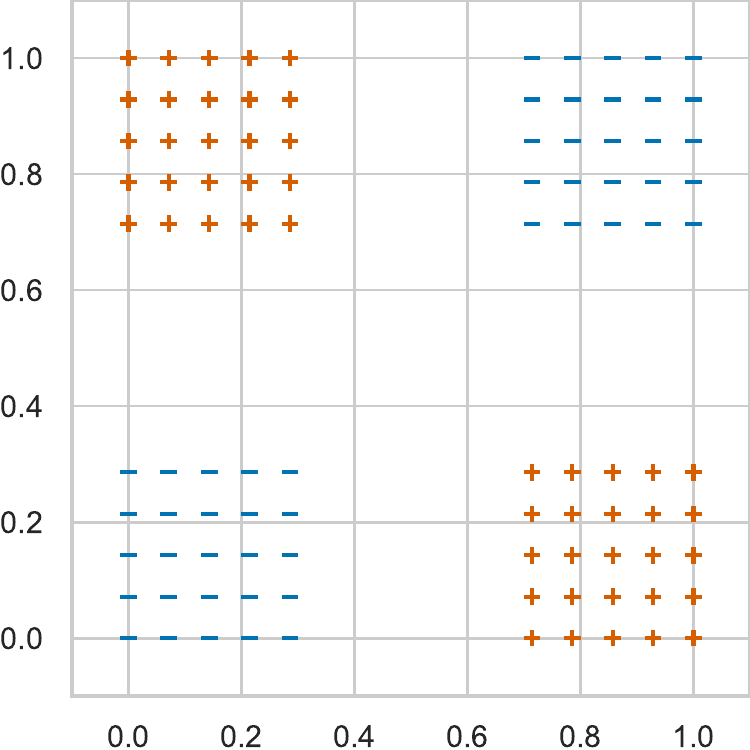}
         \caption{GROOT}
         \label{fig:bad-example-groot}
     \end{subfigure}
     \hfill
     \begin{subfigure}[b]{.15\textwidth}
         \centering
         \includegraphics[width=\textwidth]{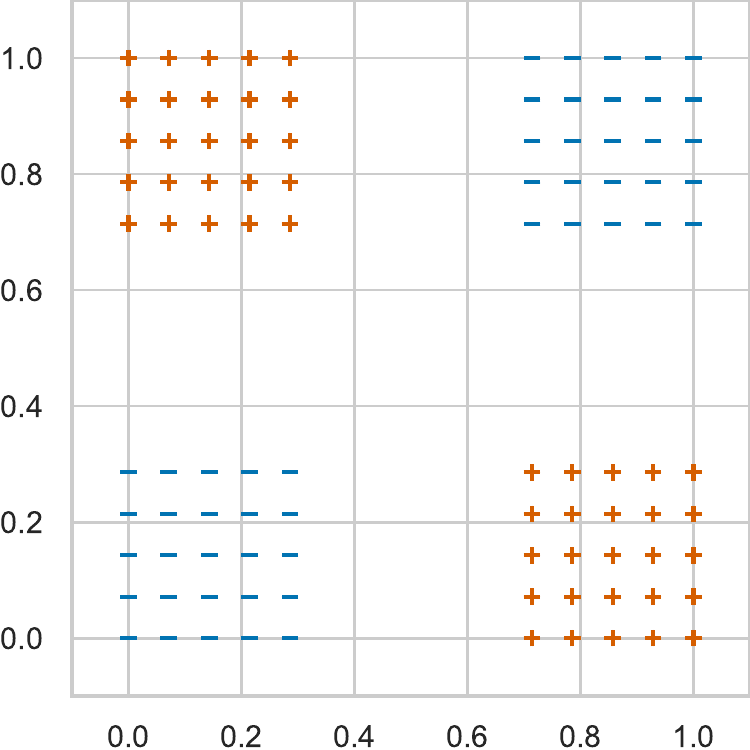}
         \caption{TREANT}
         \label{fig:bad-example-treant}
     \end{subfigure}
     \hfill
     \begin{subfigure}[b]{.15\textwidth}
         \centering
         \includegraphics[width=\textwidth]{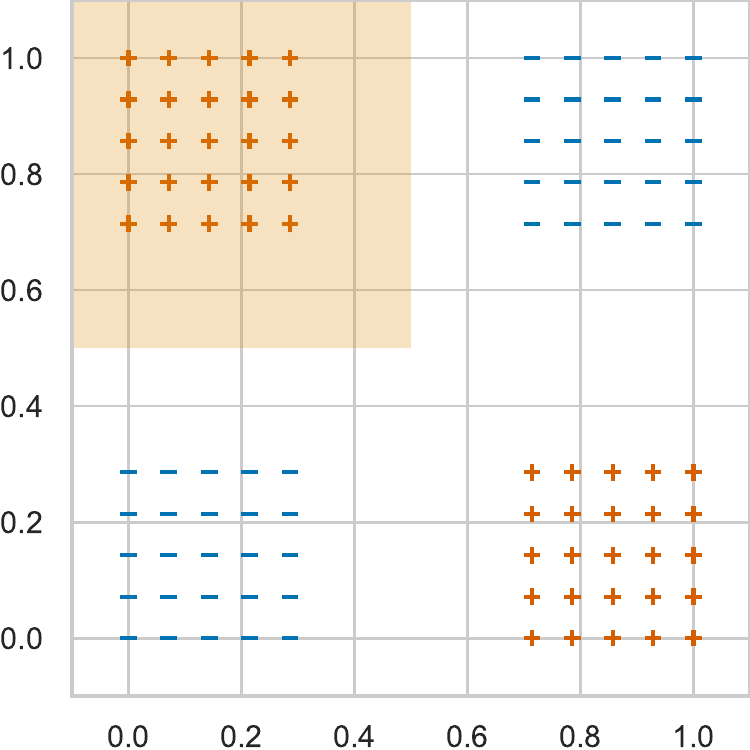}
         \caption{ROCT (MILP)}
         \label{fig:bad-example-roct}
     \end{subfigure}
    \caption{Model prediction regions on a worst case example for greedy tree learning algorithms (XOR dataset with $\epsilon{=}0.1$). Heuristic methods GROOT and TREANT cannot find any good splits, ROCT finds the optimal solution.}
    \label{fig:bad-example}
\end{figure}

\subsection{Predictive Performance on Real Data}
To demonstrate the practical performance of ROCT we compared the scores of ROCT, GROOT and TREANT on eight datasets. For each dataset we used an 80\%-20\% train-test split. To limit overfitting it is typical to constrain the maximum depth of the decision tree. To this end we select the best value for the maximum depth hyperparameter using 3-fold stratified cross validation on the training set. In each run, every algorithm gets 30 minutes to fit. For MILP, binary-MILP and LSU-MaxSAT this means that we stop the solver and retrieve its best solution at that time. The methods GROOT, TREANT and RC2-MaxSAT cannot return a solution when interrupted. Therefore when these algorithms exceed the timeout we use a dummy classifier that predicts a constant value. The adversarial accuracy scores were determined by testing for each sample whether a sample with a different label intersects its perturbation range.

Table \ref{tab:results} shows the aggregated results over these 8 datasets, all individual results are displayed in the appendix. The overall best scores were achieved with the MILP-warm method which is the MILP formulation with continous variables for thresholds and is warm started with the tree produced by GROOT.
The LSU-MaxSAT method also performed well and runs without reliance on trees trained with GROOT. TREANT's scores were lower than expected which can be attributed to the number of time outs.

\subsection{Runtime}
An advantage of using optimization solvers for training robust decision trees is that most solvers can be early stopped to output a valid tree. In figure \ref{fig:runtime-behavior} we plotted the mean training scores over all datasets for trees of depth 3 of the solvers that can be stopped. We see that all algorithms converge to nearly the same value given enough time. Moreover we find that LSU-MaxSAT quickly achieves good scores where it takes MILP-warm and Binary-MILP-warm approximately 10 and 100 seconds to catch up. The MILP-based methods that were not warm started with GROOT took approximately 1000 seconds to catch up with LSU-MaxSAT.

\begin{figure}[tb]
     \centering
     \includegraphics[width=\linewidth]{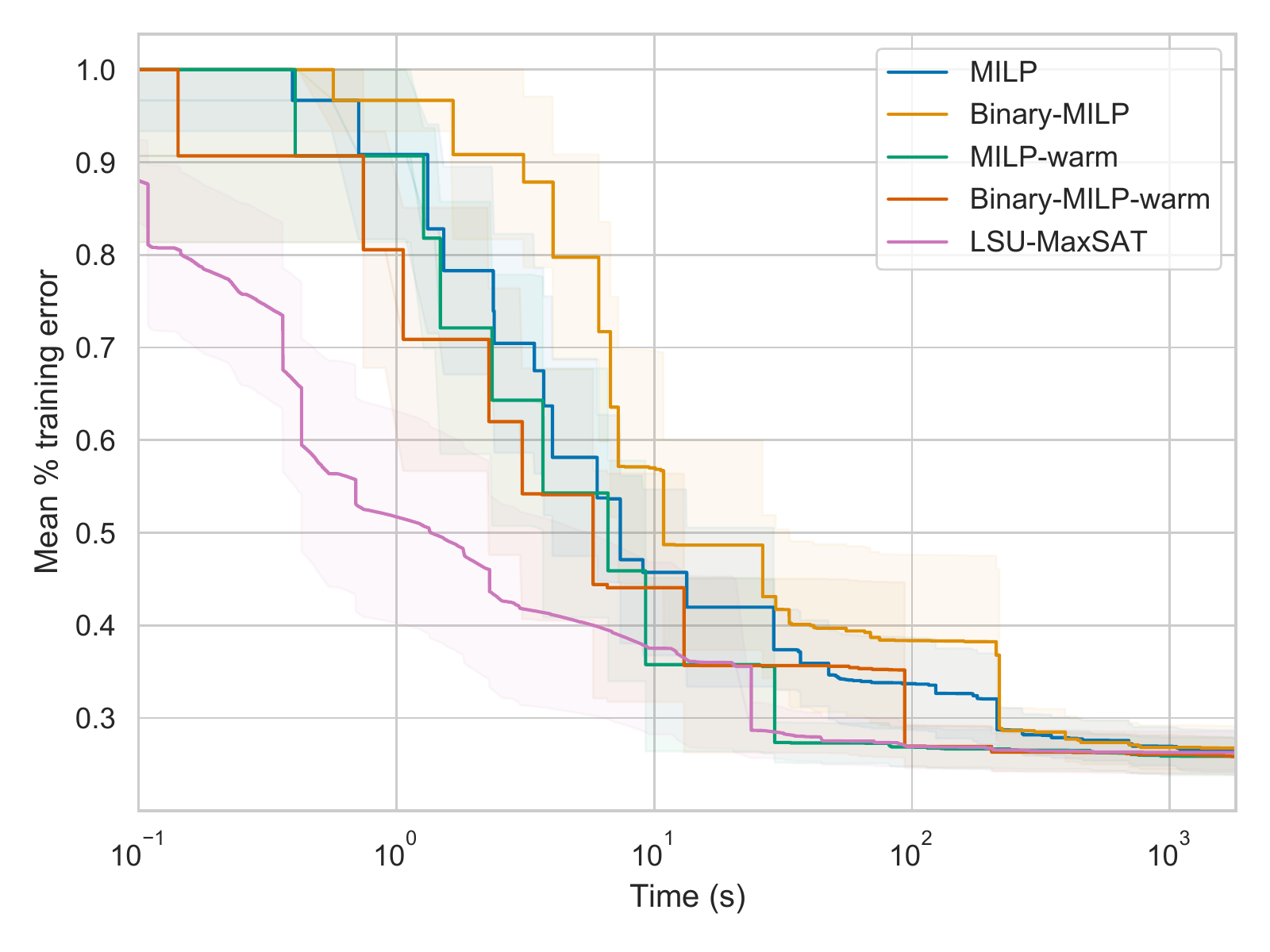}
    \caption{Mean percentage of misclassified training samples of all 8 datasets over time for trees of depth 3. The ranges represent one standard error. LSU-MaxSAT is faster at first but after 30 minutes the other methods catch up.}
    \label{fig:runtime-behavior}
\end{figure}

\subsection{Optimality}
Existing robust decision tree learning algorithms such as TREANT and GROOT have no performance guarantees. Using the LSU-MaxSAT solver we can find trees and prove their optimality on the training set which allows us to compare the scores of the heuristics with these optimal scores. In Figure \ref{fig:optimality} we plot the approximation ratios of GROOT trees after 2 hours of training. Although LSU-MaxSAT was not able to prove optimality for many datasets after a depth of 2 we can still see that GROOT scores close to optimal. All but one tree scores within a ratio of 0.92 with only one case having a ratio of approximately 0.87. We also plot the ratio between our upper bound and optimal trees. Interestingly, optimal trees of depths 1 and 2 already score close to the upper bounds in some cases.

\begin{figure}[tb]
     \centering
     \includegraphics[width=0.95\linewidth]{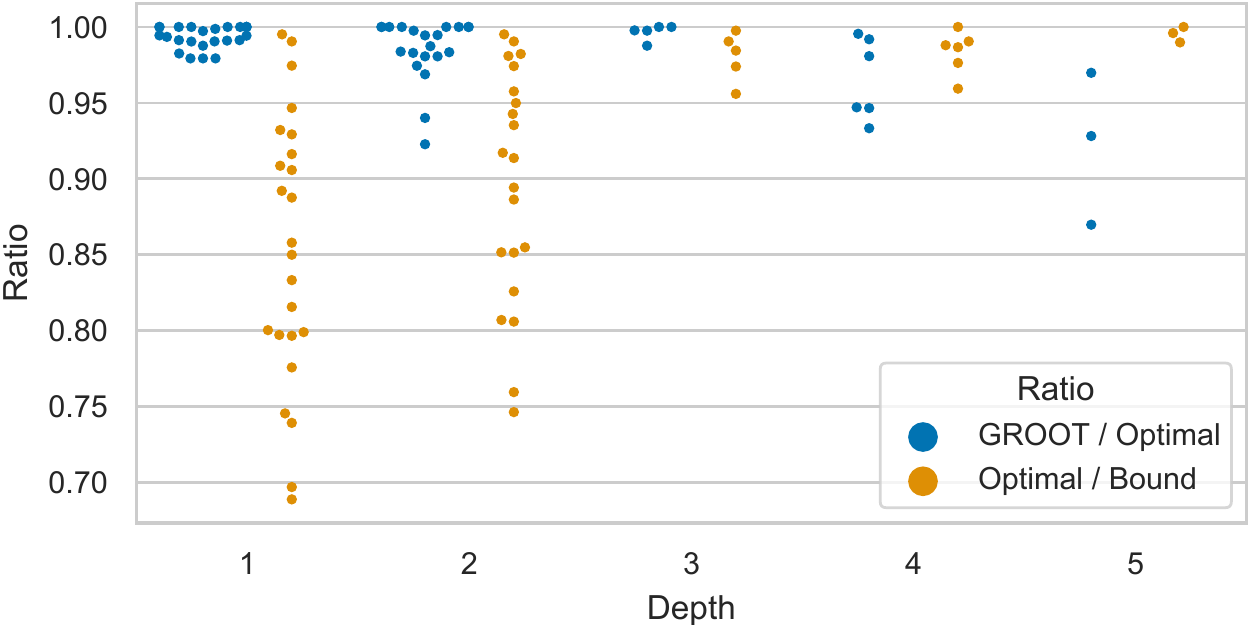}
    \caption{Ratios of training adversarial accuracy scores between GROOT vs LSU-MaxSAT's optimal trees and optimal trees vs our bound (Theorem 2). In most cases GROOT performs within 5\% of optimal. In some cases trees of depth 1 or 2 already score as well as the upper bound.}
    \label{fig:optimality}
\end{figure}

\begin{table}[tb]
    \centering
    \caption{Aggregate results over 8 datasets, means are shown with their standard error. All trees trained for 30 minutes.}
    \begin{tabular}{l|ccc}
    \toprule
    {} & Mean adv. & Mean rank & Wins \\
    Algorithm & accuracy & & \\
    \midrule
    TREANT & .692 \tiny $\pm$ .013 & 5.167 \tiny $\pm$ .604 & 7 \\
    Binary-MILP & .714 \tiny $\pm$ .013 & 3.958 \tiny $\pm$ .576 & 10 \\
    MILP & .720 \tiny $\pm$ .015 & 2.917 \tiny $\pm$ .454 & 12 \\
    RC2-MaxSAT & .724 \tiny $\pm$ .014 & 2.667 \tiny $\pm$ .393 & 10 \\
    GROOT & .726 \tiny $\pm$ .015 & 2.375 \tiny $\pm$ .450 & 16 \\
    Binary-MILP-warm & .726 \tiny $\pm$ .015 & 2.083 \tiny $\pm$ .399 & 16 \\
    LSU-MaxSAT & .729 \tiny $\pm$ .014 & 2.125 \tiny $\pm$ .303 & 13 \\
    MILP-warm & \textbf{.735} \tiny $\pm$ .015 & \textbf{1.583} \tiny $\pm$ .225 & \textbf{17} \\
    \bottomrule
    \end{tabular}
    \label{tab:results}
\end{table}

\section{Conclusions}
In this work we propose ROCT, a new solver based method for fitting robust decision trees against adversarial examples. Where existing methods for fitting robust decision trees can perform arbitrarily poorly in theory, ROCT fits the optimal tree given enough time. Important for the computational efficiency of ROCT is the insight and proof that the min-max adversarial training procedure can be computed in one shot for decision trees (Theorem 1). We compared ROCT to existing methods on 8 datasets and found that given 30 minutes of runtime ROCT improved upon the state-of-the-art. Moreover, although greedy methods have been compared to each other in earlier works, we demonstrate for the first time that the state-of-the-art actually performs close to optimal. We also presented a new upper bound for adversarial accuracy that can be computed efficiently using maximum bipartite matching (Theorem 2).

Although ROCT was frequently able to find an optimal solution and shows competitive testing performance, the choice of tree depth strongly influences runtime. Optimality could only be proven for most datasets up to a depth of 2 and for some until depth 4. Additionally, the size of ROCT's formulation grows linearly in terms of the number of unique feature values of the training dataset.
For small datasets of up to a few 1000 samples and tens of features ROCT is likely to improve performance over state-of-the-art greedy methods.
Overall, ROCT can increase the performance of state-of-the-art heuristic methods and, due to its optimal nature and new upper bound, provide insight into the difficulty of robust learning.

In the future, we will investigate realistic use cases of adversarial learning in security such as fraud / intrusion / malware detection. We expect our upper-bound method to be a useful tool in determining the sensibility of adversarial learning problems and for robust feature selection.

\bibliography{bibliography}

\end{document}